\documentclass[mlabstract]{jmlr}

\usepackage[utf8]{inputenc}
\usepackage[T1]{fontenc}
\usepackage{hyperref}
\usepackage{url}
\usepackage{booktabs}
\usepackage{amsfonts}
\usepackage{amsmath}
\usepackage{amssymb}
\usepackage{nicefrac}
\usepackage{xcolor}
\usepackage{graphicx}

\newcommand{\R}{\mathbb{R}}
\newcommand{\bPhi}{\Phi}

\jmlrvolume{}
\firstpageno{1}
\jmlryear{}
\jmlrpages{}
\jmlrproceedings{PMLR}{\small Accepted at the NeurReps 2026 Workshop (Extended Abstract Track), NeurIPS 2026}

\title[Impossibility Result and Equinorm Remedy]{An Impossibility Result
for ETF Target Geometry, and an Equinorm Remedy for Imbalanced Analytic
Continual Learning}

\author{\Name{Berk Sezer} \Email{berk.sezer@metu.edu.tr}\\
\addr Middle East Technical University, T\"urkiye}

\editor{}

\begin{document}

\maketitle

\begingroup
\renewcommand{\thefootnote}{}
\footnotetext{Code available at \texttt{https://github.com/berksezer9/etf-targets-acil}.}
\endgroup

\begin{abstract}
Analytic continual learning (ACL) methods such as ACIL and GACL fit a
ridge-regression classifier head on frozen features via the
Sherman--Morrison--Woodbury identity, achieving exemplar-free, exact updates. A separate line of work, Progressive Neural Collapse
(ProNC), shows standard (gradient-trained) continual classifiers benefit from a
progressively expanding Simplex Equiangular Tight Frame (ETF) target. It is
natural to ask whether fusing the two ideas, exact analytic updates and
ETF targets, would yield a better closed-form continual learner. We show this is
not the case as the fusion turns out to be identical to plain one-hot ACIL. Further,
Simplex ETFs are not nested across class counts, so the literal
``progressively expand the ETF'' construction stops yielding an ETF as new tasks arrive, and it introduces
catastrophic forgetting ($4\times$ that of a classic ACIL head). We therefore
separate an ETF's two defining properties, equiangularity and equinormality,
and apply each on its own to the classifier weights rather than the
targets, where the invariance above does not hold. Equinorm post-processing of
the fitted weights recovers both tail-class and
overall average accuracy across the entire long-tailed regime
from $5{:}1$ to $100{:}1$ imbalance (e.g.\ $+32.0$ tail-accuracy points and
$+16.0$ average-accuracy points at $100{:}1$ on CIFAR-100/ViT), transferring in
sign across two backbones and two datasets, at a mild cost of $5.4$
average-accuracy points on CIFAR-100/ViT on near-balanced streams (below
$\approx 3{:}1$). Equiangular geometry, by
contrast, is exactly inert on the target side but actively harmful on the
weight side, usually losing to plain one-hot. Hence, the generally useful
part of Neural Collapse geometry in Analytic Continual Learning is equinorm post-processing of the
weights, becoming more useful as class imbalance grows; and equiangularity is at best a no-op and induces catastrophic forgetting at worst.
\end{abstract}

\section{Introduction}
\label{sec:intro}

\textbf{Neural Collapse (NC)} \citep{papyan2020prevalence} describes the
geometry at the
terminal phase of training as: within-class features collapse to their means
(NC1), class means form a Simplex ETF (NC2), the classifier aligns with the
class means (NC3), and prediction reduces to a nearest-mean rule (NC4).
\citet{pronc2026} (ProNC) harnesses NC2 for continual learning (CL) by
progressively appending new class vertices to an expanding ETF target,
trained with a combined cross-entropy, alignment, and distillation loss.
\textbf{Analytic continual learning (ACL)} \citep{zhuang2022acil,zhuang2023gacl}
instead freezes a pretrained backbone and updates a ridge-regression head
recursively via the Sherman--Morrison--Woodbury (SMW) identity as new classes
arrive, exactly and without forgetting (\S\ref{sec:prelim}). However, its regression targets are plain one-hot
vectors, satisfying none of NC2's margin properties.

The natural fusion of ProNC and ACL replaces ACL's one-hot targets with ProNC's expanding
Simplex ETF, updated recursively via SMW instead of SGD, in the hope of
combining the advantages of both. Yet, we show this fusion
cannot improve on one-hot targets because the two are equivalent.
\S\ref{sec:thm1} makes this precise and verifies it up to floating-point
machine precision on real ViT and ResNet-50 features. \S\ref{sec:nonnested} then shows the
literal ``progressively expand the ETF'' idea is actively harmful, not merely
inert, due to a staleness artifact.
Naturally, performance diverges from plain ACIL (for better or worse) where the
invariance provably cannot apply: compressed frames and non-linear
post-processing of the fitted weights. We explore these settings by conducting experiments across two datasets (CIFAR-100, TinyImageNet), and two backbones (ResNet-50, ViT), with positive results mainly emerging from the equinorm property of ETFs, specifically column-norm equalization of weights.
Column-norm equalization itself is not new \citep{kang2020decoupling,alshammari2022weightbalancing}, but to the best of our knowledge, applying it to ACL and splitting it from equiangularity has not been previously investigated (full related-work comparison in Appendix~\ref{sec:related}).

\section{Preliminaries}
\label{sec:prelim}

Let $f_\theta$ be a frozen backbone producing features $\phi=f_\theta(x)\in\R^d$.
ACIL/GACL fit a ridge head $W_t=R_t^{-1}\bPhi_{1:t}^\top Y_{1:t}$, $R_t =
\bPhi_{1:t}^\top\bPhi_{1:t}+\lambda I_d$, $Y_{1:t}$ stacking one-hot rows; this
updates recursively via SMW as each task's $\bPhi_t$ arrives, without
revisiting past features, and it is identical to solving the ridge problem in
one batch over all data seen so far. A Simplex ETF $U\in\R^{m\times K}$ is a
maximally separated unit-norm frame, defined by two conditions: $\|u_c\|=1$
and $\langle u_c,u_{c'}\rangle=-\tfrac{1}{K-1}$ for $c\neq c'$, equivalently
$U^\top U=\tfrac{K}{K-1}(I_K-\tfrac1K\mathbf{1}\mathbf{1}^\top)$. The familiar
$\sum_c u_c=\mathbf{0}$ and $m\geq K{-}1$ follow from these rather than
supplementing them (Appendix~\ref{app:proofs}). NC2 states trained classifiers spontaneously converge to this
geometry; ProNC implements it as an explicit, progressively expanded target.

\section{The invariance theorem}
\label{sec:thm1}

To fuse these approaches, we replace the standard one-hot target matrix $Y$ with $E = YA^\top$. Here, $A\in\R^{m\times K}$ is a class-encoding matrix where the $c$-th column, $a_c$, represents the target vector for class $c$ (forming the ETF frame). The model is fit using the same ridge objective and evaluated by correlating the predicted target against each vertex: $s_A(x) := (\phi^\top W_A)A$, where the fitted weights are $W_A = R^{-1}\bPhi^\top E$.

We define a fitting rule as \emph{target-linear} if its fitted weights can be expressed as $W = ME$ for some matrix $M\in\R^{d\times N}$. Crucially, $M$ may depend on the features $\bPhi$ and any hyperparameters, but it must remain independent of the targets $E$. Basic ridge regression ($M = R^{-1}\bPhi^\top$), weighted ridge regression with arbitrary per-sample weights, ACIL/GACL's SMW recursion, and RanPAC's projected variant all satisfy this target-linear property.

\begin{theorem}[Target-geometry invariance]
\label{thm:invariance}
Assume the classification head is fitted by any target-linear rule and read out as described above. Then: \emph{(i)} For every target matrix $A\in\R^{m\times K}$ and every input $x$, the predicted scores satisfy $s_A(x) = s_{\text{onehot}}(x)\,(A^\top A)$. \\\emph{(ii)} Furthermore, if the Gram matrix takes the form $A^\top A = \alpha I_K + \beta\mathbf{1}\mathbf{1}^\top$ for some constants $\alpha>0$ and $\beta\in\R$, then the predictions perfectly match the one-hot case. Specifically, $\arg\max_c\,s_A(x)_c = \arg\max_c\,s_{\text{onehot}}(x)_c$ as sets, for every $x$.
\end{theorem}

\begin{corollary}[ETF targets are a no-op]
\label{cor:etf}
Every Simplex ETF satisfies the Gram matrix condition $A^\top A = \tfrac{K}{K-1}I_K - \tfrac{1}{K-1}\mathbf{1}\mathbf{1}^\top$. Thus, Theorem~\ref{thm:invariance}(ii) applies directly by setting $\alpha=K/(K-1)>0$ and $\beta=-1/(K-1)$. Notably, the standard ETF properties---such as the zero-mean condition ($\sum_c a_c=\mathbf{0}$) and the dimensionality bound ($m\geq K{-}1$)---are mathematical \emph{consequences} of the two defining ETF conditions, rather than extra assumptions. Because $A^\top A$ is identical for every possible Simplex ETF of a given size $(K,m)$, the underlying geometry remains fixed. Consequently, using ETF targets with any target-linear head is prediction-identical to using plain one-hot ACIL. This equivalence holds true for every embedding dimension $m$, every ridge penalty $\lambda\geq 0$, at every task, and under any sample weighting.
\end{corollary}

Bypassing this target-geometry invariance\footnote{All proofs, including those for the non-nestedness proposition below, are deferred to Appendix~\ref{app:proofs}.} requires either lowering the embedding dimension ($m<K{-}1$), where no Simplex ETF exists, or using non-linear post-processing on the weights. A similar invariance applies to the weight side (Corollary~\ref{cor:polar} in Appendix~\ref{app:proofs}): projecting fitted weights onto a Simplex ETF is prediction-identical to orthonormalizing them.

\section{Non-nestedness and the staleness artifact}
\label{sec:nonnested}
\label{sec:staleness}

Corollary~\ref{cor:etf} assumes a \emph{single, correctly-constructed} ETF at
the current class count $K$. The natural literal reading of ``progressively
expand'' -- keep existing vertex directions fixed, Gram--Schmidt new vertices
into the orthogonal complement of the existing span, then renormalize -- runs
into a structural obstruction: a Simplex ETF's defining constraints are
$K$-dependent \emph{global} constraints, so growing $K$ cannot preserve
existing vertex coordinates.

\begin{proposition}[Simplex ETFs are not nested]
\label{prop:nonnested}
Let $K\geq 2$, let $A\in\R^{m\times K}$ be a Simplex ETF for $K$ classes and
$A'\in\R^{m\times(K+1)}$ a Simplex ETF for $K{+}1$ classes. Then at most one
column of $A$ is also a column of $A'$. Consequently no expansion rule that
holds two or more existing vertices fixed can yield a Simplex ETF at the
larger class count, and the bound is tight: one vertex \emph{can} always be
preserved.
\end{proposition}

Every expansion after the first therefore leaves the frame off-ETF. A fused implementation must either (i) recompute the exact ETF at every task, landing back on prediction-identical-to-ACIL by Corollary~\ref{cor:etf} (head \texttt{etf\_exact}), or (ii) leave old target blocks in their original frame (\texttt{etf\_stale}). \texttt{onehot} and \texttt{etf\_exact} match completely, while \texttt{etf\_stale} causes catastrophic forgetting, collapsing monotonically from $94.8\%$ to $1.0\%$ by task 9 on CIFAR-100 ($4\times$ the forgetting of \texttt{onehot}; per-task curves in Figure~\ref{fig:staleness}, Appendix~\ref{app:extra}). This pattern holds across datasets and backbones (Appendix~\ref{app:extra}).

\section{Experiments and Results}
\label{sec:experiments}

\textbf{Setup.} We evaluate frozen ViT-B/16-in21k ($d=768$) and ResNet-50-in1k ($d=2048$) on 10-task CIFAR-100 and TinyImageNet streams. Ridge heads use float64 and $\lambda=10^{-2}$. Long-tailed streams use an exponential per-class budget at ratio $r$ (default $r=100$). Comparisons fit all heads on identical features within a single run (5 seeds).

\textbf{Weight-Side Interventions.} Since modifying the target geometry is either inert or harmful, we explore non-linear weight-side interventions instead. First, compressed frames ($m<K-1$) incur a severe accuracy penalty regardless of the frame family. Second, separating an ETF's properties reveals that imposing equiangularity alone on the weights is consistently harmful. Projecting weights onto an ETF is prediction-identical to orthonormalizing the column-centred weights, meaning the equiangularity itself contributes nothing (details for both in Appendix~\ref{app:extra}).

\textbf{Equinorm Post-Processing.} Conversely, equinorm post-processing of the fitted weights recovers both tail-class and overall average accuracy across the long-tailed regime ($r\ge5$). At $r=100$ on CIFAR-100/ViT, plain \texttt{onehot} reaches $67.8\%\pm1.7\%$ head-class accuracy but only $7.2\%\pm0.7\%$ on tail classes. Equinorm brings tail accuracy to $39.2\%\pm4.0\%$ ($+32.0$ points) and average accuracy up $+16.0$ points, while head-class accuracy remains statistically indistinguishable from baseline (Figure~\ref{fig:sweep}). This gain is smooth from $r\approx5$ onward and transfers across backbones and datasets (full sweep in Appendix~\ref{app:extra}, Tables~\ref{tab:transfer-cifar-vit}--\ref{tab:transfer-tin-resnet}), showing a mild cost only on near-balanced streams ($-5.4$ points at $r=1$).

\begin{figure}[tb]
  \centering
  \includegraphics[width=0.62\linewidth]{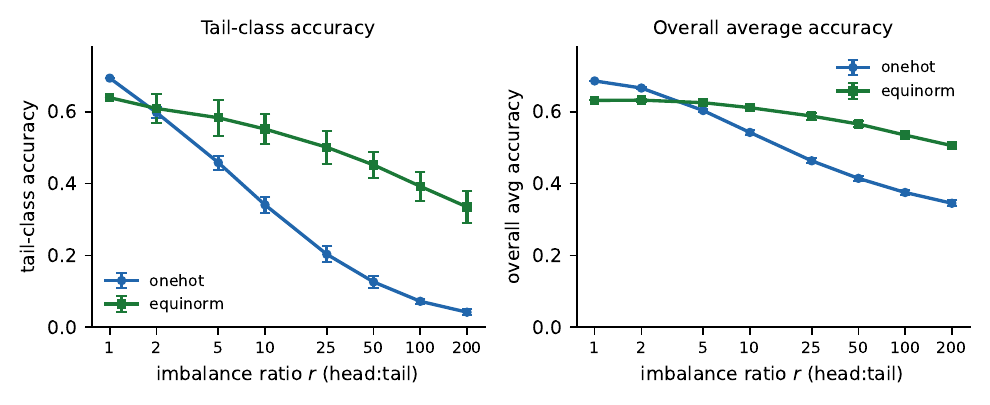}
  \caption{Imbalance-ratio sweep, $r\in\{1,\dots,200\}$, 5 seeds per ratio,
  CIFAR-100 T=10. \textbf{Left:} tail-class accuracy. \textbf{Right:} overall
  average accuracy. The crossover in overall accuracy lies between $r=2$ and
  $r=5$.}
  \label{fig:sweep}
\end{figure}

\section{Conclusion}
In summary, fusing Neural Collapse's ETF targets with analytic CL's exact updates provably changes nothing, and the literal construction's one departure---staleness in re-expanded target blocks---actively reintroduces the forgetting analytic CL exists to eliminate. Neural Collapse geometry helps only away from targets, where column-equinorm post-processing of fitted weights robustly recovers both tail and average-class accuracy in imbalanced datasets, while equal angles are usually harmful.

\bibliography{refs}

\appendix

\section{Related work}
\label{sec:related}

\textbf{Neural Collapse.} \citet{papyan2020prevalence} first characterized the
NC1--NC4 geometry. \citet{yang2022inducing} freezes (rather than learns)
an SGD-trained ETF classifier, avoiding imbalance-induced ``minority
collapse''; our contribution is instead an exact, closed-form argmax-identity
for any linear target recoding, needing no imbalance assumption or SGD.
\citet{pronc2026} expands a simplex ETF across tasks via Gram--Schmidt
extension under SGD, correcting the resulting old-class target drift with a
distillation loss (Appendix~\ref{app:relwork}); we study the same idea
distillation-free and in closed form, where that drift is either an exact
no-op or \S4's uncorrected staleness. \citet{momeni2026ananc} applies
analytic NC-compliant projections to frozen features for continual OOD
detection rather than class-incremental learning. \textbf{Analytic CL.} ACIL
\citep{zhuang2022acil} and GACL \citep{zhuang2023gacl} establish exact,
exemplar-free ridge-regression updates via the Woodbury identity; RanPAC
\citep{mcdonnell2023ranpac} prepends a frozen random projection before an
analogous head. None use equiangular targets; our theorem shows why that
would not have mattered. \textbf{Imbalanced analytic CL.} AIR
\citep{fang2024air} learns a per-class re-weighting factor at fit time; GSR
\citep{tran2026spectral} corrects tail-class spectral collapse in the Gram matrix
with an adaptive perturbation. Column-equinorm is a \emph{post-hoc} rescaling
of an already-fitted $W$, mechanistically distinct from both.
\citet{kang2020decoupling} first showed post-hoc $\tau$-normalization
($\tau{=}1$ is column-equinorm) fixes long-tail bias without retraining;
\citet{alshammari2022weightbalancing} instead correct the same imbalance with
training-time weight-decay+MaxNorm regularization, under which plain
L2-normalization alone underperforms -- a training-time sibling, not an
extension, of $\tau$-normalization. Applying $\tau$-normalization post-hoc to
a closed-form SMW head, and separating its equinorm and equiangular properties, is new;
that equiangularity specifically is the harmful half is consistent with
\citet{thrampoulidis2022imbalance}'s unconstrained-features-model result that
the imbalanced optimum departs from the exact ETF in angle as well as norm,
for any finite ridge strength (Appendix~\ref{app:extra},
\ref{app:relwork}).

\label{app:relwork}
The one-clause descriptions above are deliberately terse for space; the
remainder of this section records supporting detail on each cited work that
did not fit above.

\textbf{\citet{pronc2026} (ProNC).} Expands a simplex ETF across tasks by
Gram--Schmidt extension of the existing basis (Sec.~4.1 of that paper); this
is precisely the construction we study in \S\ref{sec:nonnested}, not a
simplified stand-in. Crucially, ProNC's own paper names the resulting
target-shift problem explicitly (``the vertices in the new ETF target that
match old classes will inevitably change after ETF expansion'') and provides
a correction for it: an alignment loss plus a knowledge-distillation loss
(their Eqs.~4--5) that stabilize old-class features across expansions, on top
of full backbone fine-tuning and (optionally empty) replay. Our staleness
result (\S4) is therefore not a bug found in ProNC itself, but a statement
about what happens to the same expanding-ETF idea with none of ProNC's
correction machinery available -- closed-form ridge regression on frozen
features has no gradient step in which a distillation loss could act. Under
that constraint the drift is either exactly undone (\texttt{etf\_exact}) or
left uncorrected (\texttt{etf\_stale}); there is no closed-form analogue of
distillation, which is precisely why the fusion collapses to those two cases.

\textbf{\citet{momeni2026ananc} (AnaNC).} Fits a closed-form ridge/ELM head to
per-class targets snapped to the nearest simplex ETF via orthogonal Procrustes
projection (their Eqs.~11--15), then performs both classification and
continual OOD detection by nearest-class-mean cosine similarity. Its target
Gram matrix satisfies exactly the $\alpha I+\beta\mathbf{1}\mathbf{1}^\top$
structure of Theorem~\ref{thm:invariance}, so its \emph{classification}
branch is a concrete instance of our identity: composing the ridge fit with
its cosine/NCM readout algebraically reduces to
$\mathrm{score}_c(x)=\alpha\, s_{\mathrm{onehot},c}(x) + \beta\sum_k
s_{\mathrm{onehot},k}(x)$, which is argmax-identical to \texttt{onehot}. Its
reported classification gains over one-hot-target baselines therefore cannot
be coming from the ETF recoding itself under matched features/$\lambda$; per
this paper's practice of comparing heads only in lockstep on identical
features, we do not draw that inference
without rerunning AnaNC ourselves, only note that the theorem predicts no
recoding-driven gain. Its \emph{OOD} gains are a different matter and are
correctly unaffected by our identity: OOD detection thresholds a continuous,
sample-dependent magnitude (the $\beta\sum_k s_{\mathrm{onehot},k}(x)$ term)
rather than taking a per-sample argmax across classes, which is exactly where
Theorem~\ref{thm:invariance}'s invariance does not apply. The broader
NC-for-OOD literature independently corroborates that this magnitude channel,
not equiangularity, is where the geometry does the work: \citet{haas2023l2norm}
show that L2-normalizing features (equinorm on the feature side) induces
earlier collapse and better OOD separation, and \citet{ammar2024neco} build a
general post-hoc OOD score directly on a sample's relative norm within the ETF
subspace -- the same magnitude channel AnaNC's classifier also thresholds.
Conversely, \citet{ray2026goen} find by ablation that CenterLoss, which forces
tighter feature compactness (the OOD analogue of snapping angles to the
constant ETF value), \emph{degrades} OOD AUROC despite improving
classification accuracy. That split -- norm-equalization helps, forced
angular rigidity hurts -- is the same equinorm/equiangular pattern this paper
identifies, recurring independently in a different task (OOD detection, not long-tailed
classification) and a different part of the network (raw features, not a
fitted linear head).

\textbf{\citet{mcdonnell2023ranpac} (RanPAC), mechanism correction.} RanPAC's
ridge head is exact-to-joint-training and one-hot-targeted like ACIL's, but is
not built on the Woodbury/Sherman--Morrison recursion: it accumulates a Gram
matrix and a class-sum matrix additively across tasks and solves once per
query, rather than maintaining an explicit recursive matrix inverse. It is
also developed independently of ACIL (no citation between them), not as a
lineage extension. \S\ref{sec:related}'s ``analogous head'' already hedges
this correctly (same input/output behavior and one-hot-ridge objective) and
should not be read as claiming the same update mechanism.
\textbf{\citet{zhuang2023gacl} (GACL).} Extends ACIL's identical Woodbury
recursion with an additional ``Exposed Class Label Gain'' term for classes
that reappear across tasks (generalized CIL), and reduces to ACIL's solution
exactly when tasks are disjoint.

\textbf{Corollary~\ref{cor:polar}, closest related mechanism.}
\citet{markou2024guiding} solves, at every SGD step, a Riemannian
(Stiefel-manifold) nearest-ETF projection of the current penultimate-layer
feature means -- the same orthogonal-Procrustes math underlying
Corollary~\ref{cor:polar}'s derivation, but applied to features during
training (to speed convergence to Neural Collapse) rather than to a
\emph{fitted, frozen} head at inference. To the best of our knowledge, no
prior work makes an argmax/prediction-equivalence statement about projecting
an already-fitted classifier onto the nearest ETF, so Corollary~\ref{cor:polar}
appears novel. \citet{thrampoulidis2022imbalance} independently support this
project's finding that flattening a fitted classifier's angles to the
constant ETF value hurts under imbalance: within the unconstrained-features
model, they prove that under class imbalance the optimal (SELI) geometry
already departs from the ETF in \emph{both} angle and norm, and that no
finite ridge strength restores exact equiangularity -- a theoretical account
of why \texttt{equiangular} (Appendix~\ref{app:extra}) loses to \texttt{onehot} in every run
tested here.

\section{Limitations}
\label{app:limitations}

We do not retrain ProNC (SGD-based) inside our pipeline, nor benchmark
directly against AIR or GSR: our theorem concerns the closed-form fusion
itself, not a head-to-head comparison. Evaluating across two backbones and two
datasets mitigates the risk of architecture- or domain-specific artifacts but
could be extended further, and the compressed-frame and the imbalance-ratio
results are tested over a bounded range.

\section{Proofs}
\label{app:proofs}

Throughout, $\bPhi\in\R^{N\times d}$ stacks the frozen features of every sample
seen so far, $Y\in\{0,1\}^{N\times K}$ is the one-hot label matrix, and
$A\in\R^{m\times K}$ has $c$-th column $a_c$, the target vector of class $c$.
A fitting rule is \emph{target-linear} if there is $M\in\R^{d\times N}$,
depending on $\bPhi$ and the hyperparameters but not on the targets, with
fitted weights $W=ME$ for every target matrix $E$. Ridge regression is
target-linear with $M=(\bPhi^\top\bPhi+\lambda I_d)^{-1}\bPhi^\top$; weighted
ridge with per-sample weights $\omega\in\R^N_{>0}$, $\Omega=\mathrm{diag}(\omega)$,
with $M=(\bPhi^\top\Omega\bPhi+\lambda I_d)^{-1}\bPhi^\top\Omega$. ACIL/GACL's
SMW recursion is target-linear because it computes exactly the ridge $M$
incrementally -- that is the content of its exactness guarantee -- and RanPAC's
head is the same with $\bPhi$ replaced by a fixed random projection of itself.

\subsection{Two lemmas}

\begin{lemma}[Target linearity]
\label{lem:linear}
Let $W\mapsto ME$ be target-linear. Then for every $A\in\R^{m\times K}$,
$W_A := M(YA^\top) = W_{\text{onehot}}A^\top$, where $W_{\text{onehot}}:=MY$.
\end{lemma}
\begin{proof}
Matrix multiplication is associative and $M$ does not depend on its target
argument, so $M(YA^\top)=(MY)A^\top$.
\end{proof}

Lemma~\ref{lem:linear} encapsulates all assumptions regarding the fitting
rule. Nothing below refers to ridge regression, $\lambda$, the SMW recursion,
or the sample weighting again; the subsequent results depend strictly on
target linearity, so any rule linear in its targets inherits all of what
follows. A rule that
inspects the targets -- tuning $\lambda$ by cross-validation on $E$, or an
early-stopped gradient fit -- falls outside the lemma, which is the formal
reason the result does not transfer to ProNC's SGD setting.

\begin{lemma}[Order invariance]
\label{lem:order}
Let $s,s'\in\R^K$ satisfy $s'=\alpha s+\gamma\mathbf{1}$ with $\alpha>0$ and
$\gamma\in\R$. Then $s'_c \geq s'_{c'} \iff s_c \geq s_{c'}$ for all $c,c'$;
in particular $\arg\max_c s'_c=\arg\max_c s_c$ as \emph{sets}, so any
deterministic tie-break selects the same class.
\end{lemma}
\begin{proof}
$s'_c-s'_{c'}=\alpha(s_c-s_{c'})$, since $\gamma$ cancels; as $\alpha>0$ the
two differences share a sign. Ranging $c'$ over all indices gives the argmax
claim, and equalities map to equalities, so tie sets coincide.
\end{proof}

Both hypotheses are necessary: if $\alpha<0$, the order of comparisons
reverses, and replacing $\gamma\mathbf{1}$ by a vector with unequal entries shifts pairwise
differences by that vector's gaps, which can be made arbitrarily large.

\subsection{Proof of Theorem~\ref{thm:invariance}}

\begin{proof}
\emph{(i)} By Lemma~\ref{lem:linear}, $W_A=W_{\text{onehot}}A^\top$, so
$s_A(x)=(\phi^\top W_{\text{onehot}}A^\top)A = s_{\text{onehot}}(x)(A^\top A)$.
\emph{(ii)} Write $s:=s_{\text{onehot}}(x)^\top$. With $A^\top A=\alpha
I_K+\beta\mathbf{1}\mathbf{1}^\top$,
\[
  s_A(x)_c \;=\; \sum_{c'} s_{c'}(A^\top A)_{c'c}
            \;=\; \alpha\,s_c \;+\; \beta\sum_{c'}s_{c'},
\]
whose second summand is the same number for every $c$. Apply
Lemma~\ref{lem:order} with $\gamma=\beta\,\mathbf{1}^\top s$, which depends on
$x$ but not on $c$.
\end{proof}

\subsection{Proof of Corollary~\ref{cor:etf}}

\begin{proof}
Let $A$ be a Simplex ETF for $K\geq2$ classes.

\emph{Gram form.} By definition $(A^\top A)_{cc}=\|a_c\|^2=1$ and $(A^\top
A)_{cc'}=\langle a_c,a_{c'}\rangle=-\tfrac1{K-1}$ for $c\neq c'$. The matrix
$\alpha I_K+\beta\mathbf{1}\mathbf{1}^\top$ with $\alpha=K/(K-1)$ and
$\beta=-1/(K-1)$ has exactly these entries, since $\alpha+\beta=
\tfrac{K}{K-1}-\tfrac{1}{K-1}=1$ on the diagonal and $\beta=-\tfrac1{K-1}$ off
it. Hence $A^\top A=\alpha I_K+\beta\mathbf{1}\mathbf{1}^\top$, and $\alpha>0$
because $K\geq2$.

\emph{$\sum_c a_c=\mathbf{0}$ and $m\geq K{-}1$ are consequences.} From the
Gram form, $(A^\top A)\mathbf{1}=(\alpha+K\beta)\mathbf{1}=\mathbf{0}$, i.e.\
$A^\top(A\mathbf{1})=\mathbf{0}$. Thus $A\mathbf{1}$ is orthogonal to
$\mathrm{range}(A)$ while also lying in it, so $A\mathbf{1}=\mathbf{0}$.
Further, $\mathrm{rank}\,A=\mathrm{rank}\,A^\top A=K-1$, because $I_K-\tfrac1K
\mathbf{1}\mathbf{1}^\top$ is the orthogonal projector onto $\mathbf{1}^\perp$;
since $\mathrm{rank}\,A\leq m$, necessarily $m\geq K-1$. (This is the sense in
which the compressed frames of \S\ref{sec:experiments} are outside the theorem: for
$m<K{-}1$ no Simplex ETF exists.)

\emph{Independence of the realization.} $A^\top A$ is determined by $K$ alone.
Equivalently, any two Simplex ETFs of the same $(K,m)$ are related by
$A'=QA$ for some $Q\in O(m)$, and $A'^\top A'=A^\top Q^\top QA=A^\top A$.

Theorem~\ref{thm:invariance}(ii) now applies at every task and for every
target-linear fit, giving prediction-identity with one-hot ACIL.
\end{proof}

\begin{remark}[Scope and boundary cases]
The paper's positive results occupy exactly the first two of the following
three exclusions. Corollary~\ref{cor:etf} does not apply to (i) $m<K{-}1$,
where no Simplex ETF exists; (ii) non-linear post-processing of the fitted
$W$, which violates Lemma~\ref{lem:linear}; or (iii) readouts that are not a
per-sample argmax across classes -- a method thresholding a score
\emph{magnitude}, as in continual OOD detection, does see the
$\beta\mathbf{1}^\top s$ term that argmax discards (Appendix~\ref{app:relwork}).
\end{remark}

\subsection{Proof of Proposition~\ref{prop:nonnested}}

\begin{proof}
Suppose two distinct columns of $A$ also appear in $A'$, say $a_c=a'_{j}$ and
$a_{c'}=a'_{j'}$ with $c\neq c'$. The indices $j,j'$ are distinct: if $j=j'$
then $a_c=a_{c'}$, contradicting $\langle a_c,a_{c'}\rangle=-\tfrac1{K-1}\neq
1=\|a_c\|^2$. Now $\langle a_c,a_{c'}\rangle$ must equal $-\tfrac1{K-1}$ by
$A$'s defining conditions and $-\tfrac1{K}$ by $A'$'s. These are unequal for
every $K\geq2$, a contradiction. Hence at most one column is shared.

For tightness, note $O(m)$ acts transitively on the unit sphere of $\R^m$, so
for any Simplex ETFs $A,A'$ there is $Q\in O(m)$ with $Qa'_1=a_1$; and $QA'$ is
again a Simplex ETF, since the defining conditions depend only on inner
products.
\end{proof}

The construction of \S\ref{sec:nonnested} holds \emph{all} $K$ existing
vertices fixed, so for $K\geq2$ it cannot produce a Simplex ETF -- which is
why the frame is exact at the first task, where it is constructed
independently rather than incrementally expanded, and deviates from an exact
ETF at all subsequent tasks. Note the proposition proves ``all
but at most one vertex must move,'' not ``every vertex moves''; the latter is
what is observed empirically but is not forced by the definitions.

\subsection{The weight-side corollary}

The same $\alpha I+\beta\mathbf{1}\mathbf{1}^\top$ structure has a second
consequence, on the weight side: one can leave targets alone and instead
impose ETF geometry on the \emph{fitted} weights, replacing $W$ by the nearest
Simplex ETF to it. That map is non-linear in $W$, so
Theorem~\ref{thm:invariance} does not apply to it -- and yet the
equiangularity it imposes is just as inert.

\begin{corollary}[Projecting $W$ onto an ETF is orthonormalization]
\label{cor:polar}
Let $K\geq2$, $d\geq K$, let $W\in\R^{d\times K}$ be the fitted weights,
$G=\tfrac{K}{K-1}(I_K-\mathbf{1}\mathbf{1}^\top/K)$ the Simplex-ETF Gram, and
consider
\[
  \min\bigl\{\,\|W-M\|_F \;:\; M\in\R^{d\times K},\; M^\top M=G \,\bigr\}.
\]
Write $G^{1/2}=c\,(I_K-\mathbf{1}\mathbf{1}^\top/K)$ with $c=\sqrt{K/(K-1)}$, and
let $WG^{1/2}=U\Sigma V^\top$ be a thin SVD. Then \emph{(a)} $M^\star:=O^\star
G^{1/2}$ with $O^\star=UV^\top$ is a minimizer, and if
$\mathrm{rank}(WG^{1/2})=K-1$ it is the \emph{unique} one -- even though
$O^\star$ itself is never unique; and \emph{(b)} for every
$\phi$, $\phi^\top M^\star = c\,(s_{\mathrm{pol}}-\bar s_{\mathrm{pol}}\mathbf{1}^\top)$ where $s_{\mathrm{pol}}=\phi^\top O^\star$ and $\bar
s_{\mathrm{pol}}$ is its mean over classes, so
$\arg\max_c(\phi^\top M^\star)_c=\arg\max_c (s_{\mathrm{pol}})_c$. Imposing
full ETF geometry on $W$ is therefore prediction-identical to replacing $W$ by
the orthonormal polar factor of the \emph{column-centred} weights $WG^{1/2}$;
the equiangularity contributes nothing beyond that.
\end{corollary}
\begin{proof}
\emph{Step 0 ($G^{1/2}$).} $P:=I_K-\mathbf{1}\mathbf{1}^\top/K$ is symmetric and
idempotent, so $(cP)^2=c^2P=G$ and $cP\succeq0$; hence $G^{1/2}=cP$ exactly,
with $G^{1/2}\mathbf{1}=\mathbf{0}$ and $\mathrm{rank}\,G^{1/2}=K-1$. No
eigendecomposition is needed.

\emph{Step 1 (the objective is a Procrustes objective).} Every feasible $M$ has
$\|M\|_F^2=\mathrm{tr}(M^\top M)=\mathrm{tr}\,G=K$, a constant. So
$\|W-M\|_F^2=\|W\|_F^2-2\,\mathrm{tr}(W^\top M)+K$, and minimizing $\|W-M\|_F$
is equivalent to maximizing $\mathrm{tr}(W^\top M)$.

\emph{Step 2 (the feasible set).} $\{M: M^\top M=G\}=\{OG^{1/2}: O\in
\R^{d\times K},\,O^\top O=I_K\}$. For $\supseteq$,
$(OG^{1/2})^\top(OG^{1/2})=G^{1/2}O^\top OG^{1/2}=G$. For $\subseteq$, let
$M^\top M=G$. Then $\|M\mathbf{1}\|^2=\mathbf{1}^\top G\mathbf{1}=0$, so $M\mathbf{1}=\mathbf{0}$, and $\mathrm{rank}\,M=\mathrm{rank}\,G=K-1$; as $d\geq K$ there is
a unit $q\perp\mathrm{range}(M)$. Put $O:=c^{-1}M+q\mathbf{1}^\top/\sqrt K$.
Using $M^\top q=\mathbf{0}$ and $\mathbf{1}^\top\mathbf{1}=K$,
\[
  O^\top O=c^{-2}M^\top M+\mathbf{1}\mathbf{1}^\top/K
          =c^{-2}G+\mathbf{1}\mathbf{1}^\top/K=P+\mathbf{1}\mathbf{1}^\top/K=I_K,
\]
and $OG^{1/2}=(c^{-1}M+q\mathbf{1}^\top/\sqrt K)(cP)=MP+c\,q(\mathbf{1}^\top
P)/\sqrt K=M$, since $MP=M-(M\mathbf{1})\mathbf{1}^\top/K=M$ and $\mathbf{1}^\top
P=\mathbf{0}^\top$. (This step, not the PSD-ness of $G$ alone, is what justifies
the parametrization.)

\emph{Step 3 (Procrustes).} By Step 2 the problem is to maximize
$\mathrm{tr}(W^\top OG^{1/2})=\mathrm{tr}((WG^{1/2})^\top O)$ over $O^\top
O=I_K$, using that $G^{1/2}$ is symmetric. With $WG^{1/2}=U\Sigma V^\top$ set
$Z:=U^\top OV$; then $\mathrm{tr}((WG^{1/2})^\top O)=\mathrm{tr}(\Sigma Z)=
\sum_i\sigma_i Z_{ii}$. Each column of $Z$ satisfies $\|U^\top Ov_i\|\leq
\|Ov_i\|=\|v_i\|=1$, so $|Z_{ii}|\leq1$ and $\mathrm{tr}(\Sigma Z)\leq
\sum_i\sigma_i$, attained by $O=UV^\top$, for which $Z=I_K$. Hence
$O^\star=UV^\top$ and $M^\star=UV^\top G^{1/2}$.

\emph{Step 4 ($M^\star$ is unique although $O^\star$ is not).} $WG^{1/2}=cWP$
has rank at most $K-1$, so $\sigma_K=0$ always and the last column of $U$ is an
arbitrary unit vector orthogonal to the others: $O^\star$ is never unique.
Since $PW^\top WP\mathbf{1}=\mathbf{0}$, the vector $\mathbf{1}/\sqrt K$ is a right
singular vector for singular value $0$; when $\mathrm{rank}(WG^{1/2})=K-1$ that
zero is simple, so $v_K=\pm\mathbf{1}/\sqrt K$ and $P=\sum_{i<K}v_iv_i^\top$.
Then
\[
  M^\star=cO^\star P=c\sum_{i<K}(O^\star v_i)v_i^\top=c\sum_{i<K}u_iv_i^\top
        =c\,(WG^{1/2})\bigl((WG^{1/2})^\top(WG^{1/2})\bigr)^{+1/2},
\]
the partial-isometry factor of the polar decomposition, which is a function of
$WG^{1/2}$ alone; the dependency on the ambiguous $u_K$ vanishes. If instead
$\mathrm{rank}(WG^{1/2})<K-1$ the zero singular value is not simple and
$M^\star$ is genuinely non-unique, so the rank hypothesis cannot be dropped.

\emph{Step 5 (scores).} $\phi^\top M^\star=\phi^\top O^\star G^{1/2}=
c\,(\phi^\top O^\star)P=c\,(s_{\mathrm{pol}}-\bar s_{\mathrm{pol}}\mathbf{1}^\top)$, which is Lemma~\ref{lem:order}'s form with $\alpha=c>0$ and
$\gamma=-c\,\bar s_{\mathrm{pol}}$. Replacing $u_K$ by any other admissible
$u_K'$ shifts $s_{\mathrm{pol}}$ by $\pm\bigl(\phi^\top(u_K'-u_K)\bigr)\mathbf{1}^\top/\sqrt K$, a constant common to every class, so
$\arg\max_c(s_{\mathrm{pol}})_c$ is well defined despite $O^\star$ not being.
\end{proof}

\textbf{Remark.} We note two technical subtleties regarding
Corollary~\ref{cor:polar}. First, $U$ and $V$ are the
singular vectors of $WG^{1/2}=c\,W(I-\mathbf{1}\mathbf{1}^\top/K)$, the
column-centred weights, \emph{not} of $W$; the polar factor of $W$ itself is a
different matrix and is not in general prediction-equivalent to $M^\star$.
Second, the corollary is an argmax statement only -- it does not say the two
score matrices are equal, and indeed they differ by the positive factor $c$ and
the per-sample shift. Numerically, on real CIFAR-100/ViT features,
\texttt{etf\_project}'s scores match $c\,(s_{\mathrm{pol}}-\bar
s_{\mathrm{pol}})$ to a maximum residual of $8.17\times10^{-14}$ with argmax
agreement $1.000000$. Every step above is additionally verified in isolation
in the accompanying code -- symbolically where the claim is an algebraic
identity, numerically otherwise -- and each hypothesis's necessity is
confirmed by a counterexample.

\section{Extended results: compressed frames and the weight-side experiments}
\label{app:extra}

\textbf{Target Invariance and Staleness.} On real CIFAR-100 features, Corollary~\ref{cor:etf} holds up to floating-point machine precision: the maximum absolute score difference between \texttt{etf\_exact} and $s_{\text{onehot}}(A^\top A)$ is $<1.1\times10^{-14}$ across backbones, with $1.0$ argmax agreement. The corresponding staleness-induced collapse of \texttt{etf\_stale}, quantified in \S\ref{sec:nonnested}, follows the same experimental protocol.

\textbf{Staleness detail.} Table~\ref{tab:main-results} gives the full
per-dataset, per-backbone breakdown behind Figure~\ref{fig:staleness}:
\texttt{onehot}/\texttt{etf\_exact} agree completely at every row (by
construction), while \texttt{etf\_stale} collapses everywhere, more severely
when the task count increases or when there are more classes per task.

\begin{figure}[h]
  \centering
  \includegraphics[width=0.32\linewidth]{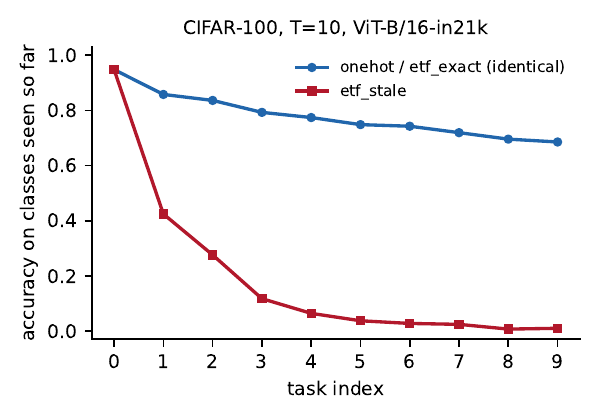}
  \caption{Per-task accuracy on classes seen so far, CIFAR-100 T=10, ViT.
  \texttt{onehot} and \texttt{etf\_exact} are identical at every
  task. \texttt{etf\_stale} ties at task 0 and then collapses as target
  staleness accumulates with drift.}
  \label{fig:staleness}
\end{figure}

\begin{table}[h]
\caption{Main CIL results, all heads fit in lockstep on identical features.}
\label{tab:main-results}
\centering
\footnotesize
\begin{tabular}{llrlrrr}
\toprule
dataset & backbone & $T$ & head & avg.\ acc.\ & avg.\ incr.\ acc.\ & forgetting \\
\midrule
CIFAR-100     & ViT       & 10 & onehot / etf\_exact & $0.6856$ & $0.7802$ & $0.0949$ \\
CIFAR-100     & ViT       & 10 & etf\_stale           & $0.0100$ & $0.1936$ & $0.3927$ \\
TinyImageNet  & ViT       & 10 & onehot / etf\_exact & $0.6675$ & $0.7354$ & $0.0829$ \\
TinyImageNet  & ViT       & 10 & etf\_stale           & $0.0050$ & $0.1773$ & $0.3557$ \\
CIFAR-100     & ResNet-50 & 20 & onehot / etf\_exact & $0.6933$ & $0.7924$ & $0.1044$ \\
CIFAR-100     & ResNet-50 & 20 & etf\_stale           & $0.0100$ & $0.1149$ & $0.2498$ \\
TinyImageNet  & ResNet-50 & 10 & onehot / etf\_exact & $0.7265$ & $0.7775$ & $0.0729$ \\
TinyImageNet  & ResNet-50 & 10 & etf\_stale           & $0.0050$ & $0.1851$ & $0.3813$ \\
\bottomrule
\end{tabular}
\end{table}

\textbf{Compressed frames detail.} Table~\ref{tab:compressed} gives the full compressed-frame
sweep, CIFAR-100 T=10, ViT backbone. \texttt{onehot} is undefined for $m<K{-}1{=}99$;
all three constructions incur a large accuracy cost relative to uncompressed
ACIL ($0.6856$), and no family shows a reliable advantage.
\texttt{truncated\_etf} builds the exact Simplex ETF in its natural $K{-}1$
dimensions, then projects it onto a random $m$-dimensional subspace and
renormalizes; we had expected it to retain some advantage from starting at
genuine ETF geometry before the dimensionality cut, but at this much
compression a random projection of a true ETF seems to have coherence close to a plain
random frame's, so little of the original structure survives.

\begin{table}[h]
\caption{Compressed-frame sweep, CIFAR-100 T=10, ViT backbone, mean $\pm$
sample stdev over 5 seeds.}
\label{tab:compressed}
\centering
\footnotesize
\begin{tabular}{lccc}
\toprule
$m$ & \texttt{welch} & \texttt{random} & \texttt{truncated\_etf} \\
\midrule
4  & $0.0182 \pm 0.0011$ & $0.0271 \pm 0.0048$ & $0.0192 \pm 0.0023$ \\
8  & $0.0445 \pm 0.0015$ & $0.0426 \pm 0.0037$ & $0.0403 \pm 0.0015$ \\
16 & $0.0631 \pm 0.0046$ & $0.0642 \pm 0.0028$ & $0.0614 \pm 0.0038$ \\
32 & $0.0780 \pm 0.0047$ & $0.0783 \pm 0.0054$ & $0.0754 \pm 0.0027$ \\
\bottomrule
\end{tabular}
\end{table}

\textbf{Equinorm and Imbalance Sweep detail.} Table~\ref{tab:sweep} gives the full imbalance-ratio
sweep behind Figure~\ref{fig:sweep} and \S\ref{sec:experiments}, CIFAR-100/ViT,
mean $\pm$ sample stdev over 5 seeds. Overall average accuracy and tail-class
accuracy move together: near balance ($r\le2$) equinorm is a net cost, but
from $r{=}5$ through $r{=}200$ it wins on both metrics simultaneously.

\begin{table}[h]
  \centering
  \caption{Full imbalance-ratio sweep, CIFAR-100/ViT, mean $\pm$ sample stdev
  over 5 seeds. \texttt{equinorm} trails \texttt{onehot} only at $r\le2$ and
  leads on both average and tail-class accuracy from $r{=}5$ onward.}
  \label{tab:sweep}
  \footnotesize
  \begin{tabular}{lrrrrrr}
    \toprule
    $r$ & \texttt{onehot} avg. & \texttt{equinorm} avg. & avg.\ $\Delta$ & \texttt{onehot} tail & \texttt{equinorm} tail & tail $\Delta$ \\
    \midrule
    1   & $0.686 \pm 0.000$ & $0.631 \pm 0.000$ & $-0.054$ & $0.693 \pm 0.000$ & $0.639 \pm 0.000$ & $-0.054$ \\
    2   & $0.666 \pm 0.002$ & $0.632 \pm 0.004$ & $-0.034$ & $0.598 \pm 0.016$ & $0.609 \pm 0.039$ & $+0.011$ \\
    5   & $0.603 \pm 0.005$ & $0.625 \pm 0.008$ & $+0.022$ & $0.458 \pm 0.020$ & $0.583 \pm 0.050$ & $+0.125$ \\
    10  & $0.542 \pm 0.008$ & $0.611 \pm 0.006$ & $+0.069$ & $0.341 \pm 0.023$ & $0.552 \pm 0.042$ & $+0.211$ \\
    25  & $0.463 \pm 0.007$ & $0.588 \pm 0.010$ & $+0.124$ & $0.203 \pm 0.022$ & $0.501 \pm 0.046$ & $+0.298$ \\
    50  & $0.415 \pm 0.006$ & $0.566 \pm 0.008$ & $+0.151$ & $0.126 \pm 0.017$ & $0.452 \pm 0.037$ & $+0.326$ \\
    100 & $0.375 \pm 0.007$ & $0.535 \pm 0.007$ & $+0.160$ & $0.072 \pm 0.007$ & $0.392 \pm 0.040$ & $+0.320$ \\
    200 & $0.345 \pm 0.008$ & $0.506 \pm 0.009$ & $+0.160$ & $0.042 \pm 0.008$ & $0.335 \pm 0.044$ & $+0.293$ \\
    \bottomrule
  \end{tabular}
\end{table}

\textbf{Weight-side Geometry detail.} Corollary~\ref{cor:polar} (Appendix~\ref{app:proofs})
establishes the algebra: imposing full ETF geometry on the fitted weights is
prediction-identical to orthonormalizing the column-centred weights, so the
equiangular half of the projection is inert. What follows is the empirical
half.

We separate an ETF's two properties with three post-processings of the same
fitted $W$: \texttt{equinorm} (equal norms, fitted angles kept),
\texttt{equiangular} (all pairwise cosines snapped to $-1/(K{-}1)$, fitted
norms restored), and \texttt{etf\_project} (both properties at once).
\texttt{etf\_project} is prediction-identical to the polar factor $UV^\top$ of
the \emph{column-centred} weights $WG^{1/2}$, confirming
Corollary~\ref{cor:polar} numerically (Appendix~\ref{app:proofs}). \texttt{equiangular} loses
to plain \texttt{onehot} on overall accuracy in all of $40$ runs on
CIFAR-100/ViT (by $6.9$--$10.6$ points), with tail-class accuracy below
\texttt{onehot}'s throughout: the fitted $W$'s pairwise cosines span
$[-0.91,+0.95]$, and flattening them to a constant $-0.0101$ discards real
class-similarity structure. Imbalance drives this almost entirely through
weight \emph{norms}, not angles: the fitted $W$'s equiangularity error
(spread of its pairwise cosines, $0$ for a true ETF) is $0.72$ at $r{=}1$
and $0.78$ at $r{=}100$, essentially unmoved, while its column-norm
ratio $\|w\|_{\max}/\|w\|_{\min}$ explodes from $5.2$ to $43.7$ over the
same range; equalizing norms alone therefore tracks imbalance, and
equalizing angles does not. \texttt{etf\_project} is nonetheless the best of
the four heads on this cell, beating \texttt{onehot} $40/40$ and
\texttt{equinorm} $40/40$. \texttt{etf\_project}'s edge over \texttt{equinorm} specifically
comes from the polar factor's columns being \emph{orthonormal}, not merely
equal-norm, so \texttt{etf\_project} additionally decorrelates class
directions that \texttt{equinorm} leaves at their fitted (correlated)
angles. Hence, its advantage here is an orthogonality effect, not an equiangularity one.

\begin{table}[h]
  \centering
  \caption{Overall average accuracy, CIFAR-100 / ViT, all four heads, full imbalance-ratio sweep (mean $\pm$ sample stdev, 5 seeds). Best per row in bold.}
  \label{tab:transfer-cifar-vit}
  \footnotesize
  \begin{tabular}{lcccc}
    \toprule
    $r$ & \texttt{onehot} & \texttt{equinorm} & \texttt{equiangular} & \texttt{etf\_project} \\
    \midrule
    1 & $0.686 \pm 0.000$ & $0.631 \pm 0.000$ & $0.580 \pm 0.000$ & $\mathbf{0.692 \pm 0.000}$ \\
    2 & $0.666 \pm 0.002$ & $0.632 \pm 0.004$ & $0.564 \pm 0.011$ & $\mathbf{0.686 \pm 0.001}$ \\
    5 & $0.603 \pm 0.005$ & $0.625 \pm 0.008$ & $0.507 \pm 0.024$ & $\mathbf{0.674 \pm 0.001}$ \\
    10 & $0.542 \pm 0.008$ & $0.611 \pm 0.006$ & $0.450 \pm 0.024$ & $\mathbf{0.659 \pm 0.002}$ \\
    25 & $0.463 \pm 0.007$ & $0.588 \pm 0.010$ & $0.381 \pm 0.018$ & $\mathbf{0.627 \pm 0.006}$ \\
    50 & $0.414 \pm 0.006$ & $0.566 \pm 0.008$ & $0.337 \pm 0.015$ & $\mathbf{0.597 \pm 0.006}$ \\
    100 & $0.375 \pm 0.007$ & $0.535 \pm 0.007$ & $0.301 \pm 0.012$ & $\mathbf{0.564 \pm 0.005}$ \\
    200 & $0.345 \pm 0.008$ & $0.505 \pm 0.009$ & $0.276 \pm 0.011$ & $\mathbf{0.529 \pm 0.007}$ \\
    \bottomrule
  \end{tabular}
\end{table}

\begin{table}[h]
  \centering
  \caption{Overall average accuracy, TinyImageNet / ViT, all four heads, full imbalance-ratio sweep (mean $\pm$ sample stdev, 5 seeds). Best per row in bold.}
  \label{tab:transfer-tin-vit}
  \footnotesize
  \begin{tabular}{lcccc}
    \toprule
    $r$ & \texttt{onehot} & \texttt{equinorm} & \texttt{equiangular} & \texttt{etf\_project} \\
    \midrule
    1 & $0.668 \pm 0.000$ & $0.601 \pm 0.000$ & $0.576 \pm 0.000$ & $\mathbf{0.677 \pm 0.000}$ \\
    2 & $0.644 \pm 0.003$ & $0.600 \pm 0.004$ & $0.545 \pm 0.008$ & $\mathbf{0.673 \pm 0.001}$ \\
    5 & $0.575 \pm 0.004$ & $0.596 \pm 0.009$ & $0.461 \pm 0.011$ & $\mathbf{0.663 \pm 0.002}$ \\
    10 & $0.507 \pm 0.005$ & $0.588 \pm 0.010$ & $0.397 \pm 0.007$ & $\mathbf{0.651 \pm 0.002}$ \\
    25 & $0.428 \pm 0.011$ & $0.575 \pm 0.012$ & $0.329 \pm 0.009$ & $\mathbf{0.625 \pm 0.003}$ \\
    50 & $0.381 \pm 0.008$ & $0.558 \pm 0.009$ & $0.298 \pm 0.012$ & $\mathbf{0.599 \pm 0.005}$ \\
    100 & $0.346 \pm 0.009$ & $0.533 \pm 0.009$ & $0.267 \pm 0.011$ & $\mathbf{0.566 \pm 0.004}$ \\
    200 & $0.318 \pm 0.007$ & $0.506 \pm 0.009$ & $0.244 \pm 0.010$ & $\mathbf{0.533 \pm 0.004}$ \\
    \bottomrule
  \end{tabular}
\end{table}

\begin{table}[h]
  \centering
  \caption{Overall average accuracy, CIFAR-100 / ResNet-50, all four heads, full imbalance-ratio sweep (mean $\pm$ sample stdev, 5 seeds). Best per row in bold.}
  \label{tab:transfer-cifar-resnet}
  \footnotesize
  \begin{tabular}{lcccc}
    \toprule
    $r$ & \texttt{onehot} & \texttt{equinorm} & \texttt{equiangular} & \texttt{etf\_project} \\
    \midrule
    1 & $0.693 \pm 0.000$ & $0.680 \pm 0.000$ & $\mathbf{0.696 \pm 0.000}$ & $0.686 \pm 0.000$ \\
    2 & $0.680 \pm 0.002$ & $0.673 \pm 0.002$ & $\mathbf{0.680 \pm 0.002}$ & $0.680 \pm 0.002$ \\
    5 & $0.638 \pm 0.002$ & $0.660 \pm 0.004$ & $0.630 \pm 0.005$ & $\mathbf{0.666 \pm 0.004}$ \\
    10 & $0.593 \pm 0.006$ & $0.645 \pm 0.004$ & $0.581 \pm 0.008$ & $\mathbf{0.650 \pm 0.004}$ \\
    25 & $0.524 \pm 0.005$ & $0.619 \pm 0.007$ & $0.511 \pm 0.007$ & $\mathbf{0.622 \pm 0.008}$ \\
    50 & $0.473 \pm 0.003$ & $0.594 \pm 0.008$ & $0.463 \pm 0.005$ & $\mathbf{0.594 \pm 0.009}$ \\
    100 & $0.431 \pm 0.006$ & $0.560 \pm 0.011$ & $0.422 \pm 0.007$ & $\mathbf{0.560 \pm 0.012}$ \\
    200 & $0.395 \pm 0.009$ & $\mathbf{0.531 \pm 0.007}$ & $0.387 \pm 0.009$ & $0.528 \pm 0.007$ \\
    \bottomrule
  \end{tabular}
\end{table}

\begin{table}[h]
  \centering
  \caption{Overall average accuracy, TinyImageNet / ResNet-50, all four heads, full imbalance-ratio sweep (mean $\pm$ sample stdev, 5 seeds). Best per row in bold.}
  \label{tab:transfer-tin-resnet}
  \footnotesize
  \begin{tabular}{lcccc}
    \toprule
    $r$ & \texttt{onehot} & \texttt{equinorm} & \texttt{equiangular} & \texttt{etf\_project} \\
    \midrule
    1 & $0.726 \pm 0.000$ & $0.711 \pm 0.000$ & $\mathbf{0.727 \pm 0.000}$ & $0.717 \pm 0.000$ \\
    2 & $0.721 \pm 0.001$ & $0.710 \pm 0.001$ & $\mathbf{0.721 \pm 0.001}$ & $0.716 \pm 0.001$ \\
    5 & $0.696 \pm 0.003$ & $0.703 \pm 0.002$ & $0.693 \pm 0.003$ & $\mathbf{0.711 \pm 0.002}$ \\
    10 & $0.667 \pm 0.004$ & $0.697 \pm 0.003$ & $0.659 \pm 0.004$ & $\mathbf{0.705 \pm 0.002}$ \\
    25 & $0.618 \pm 0.005$ & $0.687 \pm 0.004$ & $0.603 \pm 0.006$ & $\mathbf{0.694 \pm 0.002}$ \\
    50 & $0.573 \pm 0.006$ & $0.676 \pm 0.003$ & $0.555 \pm 0.010$ & $\mathbf{0.681 \pm 0.002}$ \\
    100 & $0.525 \pm 0.007$ & $0.660 \pm 0.003$ & $0.506 \pm 0.012$ & $\mathbf{0.664 \pm 0.004}$ \\
    200 & $0.480 \pm 0.008$ & $0.634 \pm 0.003$ & $0.466 \pm 0.012$ & $\mathbf{0.636 \pm 0.002}$ \\
    \bottomrule
  \end{tabular}
\end{table}

Repeating it for the other three cells of a
$2{\times}2$ backbone$\times$dataset grid
(full sweep in Tables~\ref{tab:transfer-cifar-vit}--\ref{tab:transfer-tin-resnet})
splits the two claims apart (with $8$ ratios $\times\,5$ seeds each). \textbf{Equinorm} is robust: at
$r{=}100$ the equinorm-over-\texttt{onehot} gain is $+16.0$ to $+18.7$ points
on the two ViT cells and $+12.9$ to $+13.5$ on the two ResNet-50 cells. \textbf{Equiangularity} is consistently detrimental on both backbones, but at
very different scales -- severe on ViT, mild on ResNet
($-6.9$ to $-11.4$ points on ViT vs.\ $-0.3$ to $-1.9$ on ResNet). \texttt{etf\_project} beats \texttt{equinorm} by $+2.3$--$7.5$
points on ViT but only $+0.2$--$0.8$ on ResNet-50 (and trails it at
CIFAR-100/ResNet-50, $r{=}200$), and on balanced ResNet-50 streams it loses
outright to plain \texttt{onehot}, reversing the uniform advantage seen on ViT.
Corollary~\ref{cor:polar} itself is unaffected -- it is an algebraic identity
and holds in all four cells by construction. This is consistent with ViT's class
directions being far more mutually correlated than ResNet-50's, so
decorrelating them (what \texttt{etf\_project} does beyond \texttt{equinorm})
has more room to help on ViT specifically. In short, equiangularity never helps and usually hurts in every cell tested;
how much orthonormality yields beyond equinorm is backbone-dependent, and
equinorm alone is the one intervention that both helps and generalizes across backbones.

\textbf{Clustering by equinorm status.} The four
heads do not just each move in a documented direction with $r$ -- they
form a 2-way cluster based on equinorm status. Define the \emph{cluster gap} as
$\min(\texttt{equinorm},\texttt{etf\_project}) -
\max(\texttt{onehot},\texttt{equiangular})$: positive means the two equinorm
heads both outrank both non-equinorm heads. Table~\ref{tab:cluster} shows
this gap is negative at $r{\le}2$ in every cell and turns positive between
$r{=}2$ and $r{=}5$ -- the same crossover already noted for the
equinorm-over-\texttt{onehot} gain (Figure~\ref{fig:sweep}) -- after which
it grows monotonically to $+0.13$ to $+0.19$ by $r{=}200$. Below the
crossover the equinorm/non-equinorm split is not yet the operative axis: at
$r{=}1$ the four heads span $11.2$ and $10.1$ points on the two ViT cells
but only $1.6$ points on both ResNet-50 cells, so ``head choice is
irrelevant at balance'' is a ResNet-50 statement, not a general one -- on
ViT the ranking at balance is instead set by \texttt{etf\_project}'s
decorrelation edge over \texttt{onehot}, which has nothing to do with
equinorm status. Past the crossover, equinorm status alone explains the
ranking in all four cells, consistent with the norm-driven mechanism above.

\begin{table}[h]
  \centering
  \caption{Cluster gap
  $\min(\texttt{equinorm},\texttt{etf\_project}) -
  \max(\texttt{onehot},\texttt{equiangular})$: negative means equinorm
  status does not yet explain the ranking; positive means the two equinorm
  heads both outrank both non-equinorm heads. Positive for every ratio from
  $r{=}5$ through $r{=}200$.}
  \label{tab:cluster}
  \footnotesize
  \begin{tabular}{lccc}
    \toprule
    cell & $r{=}2$ & $r{=}5$ & $r{=}200$ \\
    \midrule
    CIFAR-100 / ViT & $-0.034$ & $+0.022$ & $+0.160$ \\
    TinyImageNet / ViT & $-0.044$ & $+0.022$ & $+0.188$ \\
    CIFAR-100 / ResNet-50 & $-0.007$ & $+0.022$ & $+0.133$ \\
    TinyImageNet / ResNet-50 & $-0.011$ & $+0.007$ & $+0.154$ \\
    \bottomrule
  \end{tabular}
\end{table}

\end{document}